\documentclass{aei} 
\pdfoutput=1
\usepackage[T1]{fontenc}
\usepackage{multicol,cite,url,color}
\usepackage{supertabular}
\usepackage{mathptmx} 
\usepackage{amsmath}
\usepackage{amsmath}
\usepackage{amssymb}
\usepackage{algorithm}
\usepackage{amsthm}
\usepackage{mathtools}
\usepackage{algorithm}
\usepackage{algpseudocode}
\usepackage{booktabs}

\usepackage{soul}
\definecolor{softyellow}{RGB}{255,245,180}  
\sethlcolor{softyellow}

\usepackage{upgreek} 
\usepackage{ifpdf}
\ifpdf \usepackage[pdftex]{graphicx}
\else\usepackage[dvips]{graphicx}\fi
\DeclareGraphicsExtensions{.pdf,.png,.jpg,.mps,.eps}
\graphicspath{{figures/}}
\newtheorem{theorem}{Theorem}[section]
\newtheorem{definition}{Definition}[section]

\newtheorem{remark}{Remark}[section]

\usepackage[pdftex,unicode,bookmarks=false,hidelinks]{hyperref}
\hypersetup{%
baseurl={http://www.aei.tuke.sk/},
pdfcreator={pdfcsLaTeX},
pdfkeywords={aei, tuke},
pdftitle={MULTI-AGENT ROUTE PLANNING AS A QUBO PROBLEM},
pdfauthor={Mgr. Ren\'ata Rusn\'akov\'a,
http://people.tuke.sk/renata.rusnakova},
pdfsubject={Multi-Agent Route Planning as a QUBO Problem}
}
\begin{document}
\pagerange{1}{2}
\aeino{Acta Electrotechnica et Informatica, Vol. 23, No. X, 2023}
\doino{DOI: xx.xxxx/xxxxxx-xxx-xxxx-x}
\title{Multi-Agent Route Planning as a QUBO Problem}
\headtitle{%
Multi-Agent Route Planning as a QUBO Problem
}
\author{
Ren\'ata RUSN\'AKOV\'A$^*$ 
, Martin CHOVANEC$^*$ 
and Juraj GAZDA$^*$ 
}

\adresses{
$^*$Department of Computers and Informatics, Faculty of Electrical
Engineering and Informatics, Technical University of Ko\v{s}ice,
Letn\'a~9, 042\,00~Ko\v{s}ice, E-mail:
renata.rusnakova@tuke.sk\\
}
\babstract 
Multi-Agent Route Planning considers selecting vehicles, each associated with a single predefined route, such that route-level coverage utility is maximized while redundant spatial overlaps are limited.
This paper gives a formal problem definition, proves NP-hardness by reduction from the Weighted Set Packing problem, and derives a Quadratic Unconstrained Binary Optimization formulation whose coefficients directly encode route utility rewards and pairwise overlap penalties.
A single penalty parameter $\lambda$ controls the coverage--overlap trade-off. We distinguish between a soft regime, which supports multi-objective exploration, and a hard regime, in which the penalty is strong enough to effectively enforce near-disjoint routes.
We describe a practical pipeline for generating city instances, constructing candidate routes, building the QUBO matrix, and solving it with a binary quadratic programming baseline (Gurobi), simulated annealing, and D-Wave hybrid quantum annealing.
Experiments on Barcelona instances with up to $10{,}000$ vehicles reveal a clear coverage--overlap knee and show that Pareto-optimal solutions are mainly obtained under the hard-penalty regime, while D-Wave hybrid solvers and Gurobi achieve very similar objective values on matching configurations with only minor runtime differences as problem size grows.
\eabstract 
\begin{multicols}{2}[][0pt]

\section{INTRODUCTION}
Modern cities face the challenge of coordinating multiple vehicles, such as taxis or delivery vans, to cover as much territory as possible without unnecessary redundancy. In the Multi-agent Route Planning (MaRP) problem considered here, each vehicle in a fleet is associated with a fixed route through a road network, and the central decision is which vehicles should be selected. The optimization model rewards the route-level coverage utility of selected vehicles and penalizes spatial overlap between their routes. Actual network coverage, measured as the union of nodes reached by the selected routes, is then evaluated after optimization as a separate performance metric. Deploying more vehicles usually increases coverage, but also raises the risk of overlap, leading to redundant traversal and potential congestion. This trade-off is central to MaRP and closely related to set packing problems in combinatorial optimization \cite{Hochbaum1982, Hoffman2024}, from which we later establish a polynomial-time reduction showing NP-hardness of the general set-based formulation.

The MaRP problem lies at the intersection of multi-agent path planning and coverage optimization. 
Related work in multi-agent pathfinding (MAPF) focuses primarily on collision avoidance and joint cost minimization under temporal constraints, with algorithms such as Conflict-Based Search providing optimal solutions for coordinated motion planning \cite{Sharon2015, Stern2019}. 
In coverage-oriented settings, multi-vehicle planning methods aim to maximize spatial reach while reducing redundant traversal and energy consumption, particularly in robotic and cooperative systems \cite{Galceran2013}. 
Recent studies have also explored quantum and hybrid optimization techniques for multi-vehicle coverage and routing problems, highlighting both the NP-hard nature of these tasks and the potential of tailored quantum heuristics \cite{Rao2022}. 
In contrast to MAPF, MaRP does not model time windows or collision avoidance explicitly; instead, it interprets route overlap as a proxy for congestion and regulates the coverage-overlap trade-off through a tunable penalty parameter.
These works motivate our focus on coordinated route selection and overlap minimization in MaRP, while our QUBO-based formulation and solver evaluation provide a complementary optimization perspective.

We model the MaRP vehicle selection problem as a binary optimization task, where each candidate vehicle corresponds to a binary decision variable indicating whether its route is selected. The objective function rewards a precomputed route-level coverage utility and penalizes pairwise route overlaps, with a tunable parameter $\lambda$ controlling the balance between these competing goals. This structure induces quadratic interactions between decision variables through shared network elements, naturally fitting the Quadratic Unconstrained Binary Optimization (QUBO)/Ising formulation \cite{Lucas2014, Glover2019, Kadowaki1998}.

In this article, we provide a formal definition of the MaRP problem and reformulate it as a QUBO. Our solution pipeline includes generating the road network from OpenStreetMap, assigning each vehicle a single precomputed route obtained from the Valhalla routing engine, and constructing the corresponding QUBO coefficient matrix. Experiments are conducted on the Barcelona city center, where we apply multiple solvers: Gurobi as a binary quadratic programming baseline, simulated annealing as a classical heuristic, and D-Wave's hybrid quantum annealer. We calculate and systematically explore both soft and hard $\lambda$ parameter regimes to evaluate the resulting coverage--overlap trade-off under different penalty settings. To evaluate solutions, we define several metrics that quantify realized coverage, overlap distribution, and post-optimization graph properties.

The key contributions of this article are:
\begin{itemize}
\item We provide a formal definition of the MaRP problem and establish its NP-hardness via a polynomial-time reduction from the Weighted Set Packing problem (WSP).
\item We derive a QUBO formulation of MaRP with interpretable coefficients that explicitly encode route-level coverage rewards and pairwise overlap penalties, and distinguish between soft ($\lambda_{\text{soft}}$) and hard ($\lambda_{\text{hard}}$) penalty regimes for exploring the coverage--overlap trade-off.
\item We design a reproducible experimental pipeline using OpenStreetMap road networks and Valhalla-generated vehicle routes, and evaluate the model on large-scale instances from the Barcelona city center.
\item We present a comparative empirical study of Gurobi, heuristic, and hybrid quantum solvers, showing that hybrid quantum annealing achieves objective values very close to the Gurobi baseline on matching configurations, with runtime scaling that remains practical for the tested problem sizes.
\end{itemize}

The remainder of the paper is organized as follows. Section~\ref{sec:problem} formally defines the MaRP problem and its inputs and outputs. Section~\ref{sec:np} proves the NP-hardness of MaRP. Section~\ref{sec:qubo} introduces the QUBO formulation and the role of the penalty parameters $\lambda_{\text{soft}}$ and $\lambda_{\text{hard}}$. Section~\ref{sec:simulation} describes implementation details, evaluation metrics, and solver characteristics. The experimental results are provided in Section~\ref{sec:results}, and finally Section~\ref{sec:conclusion} discusses practical applications and concludes the paper.

\section{Problem Definition}
\label{sec:problem}
MaRP focuses on the coordination of a fleet of vehicles moving through a shared network, with the goal of ensuring that the system as a whole achieves efficient coverage while avoiding unnecessary redundancy. Unlike problems where new routes must be designed, in MaRP each vehicle is typically assumed to have a predefined path based on operational constraints, such as a given origin and destination. The central decision is therefore not how each vehicle should travel, but rather \textit{which vehicles should be selected} so that their combined activity yields the greatest benefit for the network.

Formally, let us consider $n$ vehicles indexed by $i = 1,\dots,n$. Each vehicle $i$ follows a predefined route covering a subset of the network's elements $e \in U$ (nodes), denoted $S_i \subseteq U$, where $U$ is the set of all relevant locations in the network. We define the following quantities:
\begin{itemize}
\item $u_i$: the \emph{coverage reward} assigned to vehicle $i$ is defined as the individual contribution of route $S_i$ to the network coverage before interactions with other selected routes are considered: $u_i = |S_i|$, where $S_i$ contains the distinct network nodes visited by the route.\footnote{Note that $\sum_i u_i x_i$ may count a shared node multiple times; the overlap penalty term partially corrects for this redundancy in the objective, while actual set-union coverage is measured separately as a post-optimization metric (see Section~\ref{subsec:metics}).}
\item $c_{ij}$: the \emph{overlap} between vehicle $i$ and $j$, defined as the number of elements common to routes of vehicles $i$ and $j$: $c_{ij} = |S_i \cap S_j|$ for $i \neq j$. This represents redundant spatial coverage provided by vehicles $i$ and $j$ if both are selected.\footnote{The overlap term is based on pairwise route interactions. As a result, the formulation does not distinguish between different higher-order patterns of overlap, such as many routes sharing one node versus several pairwise overlaps distributed across different nodes. This pairwise representation keeps the model tractable, but it may not fully capture higher-order congestion effects.}
\end{itemize}

Each vehicle $i$ is associated with a binary decision variable $x_i \in \{0,1\}$, where $x_i = 1$ means vehicle $i$ and its route $S_i$ are selected, and $x_i = 0$ means they are not used. The MaRP objective can then be expressed as maximizing route-level coverage reward minus a penalty for route overlap:
\begin{equation}
\label{obj}
\operatorname*{max}_{\mathbf{x} \in \{0,1\}^n}
\left(
\sum_{i=1}^{n} u_i x_i
-
\lambda \sum_{1 \le i < j \le n} c_{ij} x_i x_j
\right),
\end{equation}
where $\lambda > 0$ is a weighting parameter that controls the trade-off between rewarding useful routes and penalizing overlap. The first term $\sum_i u_i x_i$ represents a precomputed route utility of the chosen vehicles, and the second term $\lambda \sum_{i<j} c_{ij} x_i x_j$ subtracts a penalty for each unit of overlap between any two chosen vehicles. By tuning $\lambda$, we can emphasize one objective over the other: a larger $\lambda$ penalizes overlaps more severely, favoring solutions with little to no overlap, whereas a smaller $\lambda$ puts relatively more weight on route utility and allows more redundancy.
The solution to the MaRP problem is an optimal binary selection vector $\mathbf{x}^* = \left( x_1^*,\dots,x_n^*\right)$, $x_i^* \in \{0,1\}$, that maximizes the defined objective in Eq.~\eqref{obj}.

If each vehicle has only a single route, then this captures the scenario completely. In cases where each vehicle has $k$ possible routes, indexed locally by $r = 1,\dots,k$ for each vehicle $i$, the model can be extended by introducing a binary variable $x_{i,r}$ for each route option $(i,r)$ and adding a constraint that at most one route per vehicle is selected:
\begin{equation}
\sum_{r=1}^{k} x_{i,r} \leq 1, \quad \forall i.
\end{equation}
This is an at-most-one route selection constraint. If exactly one route must be chosen for every vehicle, the equality form $\sum_{r=1}^{k} x_{i,r}=1$ gives the usual one-hot or one-of-$k$ constraint commonly used in QUBO formulations of assignment and routing problems~\cite{Lucas2014,Glover2019}. Such constraints can be incorporated into QUBO models as quadratic penalty terms. In the development below, we assume for simplicity that each vehicle has a single candidate route.

\section{NP-Hardness via Reduction from Weighted Set Packing}
\label{sec:np}
MaRP belongs to a class of combinatorial optimization problems that are computationally challenging. We now prove that selecting an optimal set of vehicles in MaRP is NP-hard by giving a polynomial-time reduction from the Weighted Set Packing (WSP) problem \cite{Karp1972, Garey1979}. WSP is a well-known NP-hard problem (a generalization of the NP-complete Set Packing problem) defined as follows:

\begin{definition}[Weighted Set Packing]
\label{def:WSP}
Given a universe $U$ and a family of sets $S_1,\dots,S_m \subseteq U$ with weights $w_i \ge 0$ for $i=1,\dots,m$, the \emph{Weighted Set Packing} problem asks to select a subfamily of pairwise disjoint sets maximizing the total weight. Formally, introducing binary variables $x_i \in \{0,1\}$ indicating whether set $S_i$ is selected, the goal is to
\begin{equation}
\label{wsp}
\begin{aligned}
\max \quad & \sum_{i=1}^m w_i x_i \\
\text{s.t.} \quad 
& x_i \in \{0,1\}, \\
& S_i \cap S_j = \emptyset \quad \text{whenever } x_i = x_j = 1 .
\end{aligned}
\end{equation}
\end{definition}


\begin{theorem}\label{thm:nphard}
The MaRP optimization problem in Eq.~\eqref{obj} is NP-hard.
\end{theorem}

\begin{proof}
Let $(U,\{S_i\}_{i=1}^m,\{w_i\}_{i=1}^m)$ be an arbitrary instance of WSP with $w_i \ge 0$.

We construct a corresponding MaRP instance with $n = m$ candidate vehicles. Each vehicle $i$ is assigned a route covering exactly the elements of $S_i$, and its coverage reward is set to $u_i := w_i$. For every pair $i \neq j$, we define the overlap coefficient as $c_{ij} := |S_i \cap S_j|$. The penalty parameter is chosen as
$\lambda := 1 + \sum_{k=1}^m w_k,$
which is strictly larger than the total route reward that can be obtained by selecting all vehicles.
Consider any feasible MaRP solution $\mathbf{x} \in \{0,1\}^m$. Suppose that $\mathbf{x}$ selects two vehicles $i \neq j$ such that $x_i = x_j = 1$ and $c_{ij} \ge 1$, i.e., their routes overlap. We show that such a solution cannot be optimal.
If we remove vehicle $i$ from the solution (i.e., set $x_i = 0$), the coverage term decreases by at most $u_i = w_i \le \sum_{k=1}^m w_k$. On the other hand, the overlap penalty decreases by at least $\lambda$, since at least one overlapping pair is eliminated. By the choice of $\lambda$, the decrease in penalty strictly outweighs the loss in coverage, and the objective value strictly improves. Therefore, any solution containing overlapping routes cannot be optimal, i.e., every optimal MaRP solution must select only vehicles with pairwise disjoint routes. Under this condition, all overlap terms vanish, and the objective function reduces to
$\sum_{i=1}^m u_i x_i = \sum_{i=1}^m w_i x_i,$
which is exactly the objective of the WSP problem.

Thus, any optimal solution of the constructed MaRP instance corresponds directly to an optimal solution of the original WSP instance. Since the transformation can be carried out in polynomial time and WSP is NP-hard, it follows that MaRP is NP-hard.
\end{proof}

\begin{remark}
The reduction proves NP-hardness for the general set-based MaRP model, where routes are treated as subsets of network elements. We use a more structured version, where routes are valid paths in an urban road network. Although this road-network variant is more specific, it keeps the same route-selection objective with pairwise overlap penalties. The reduction therefore supports the computational difficulty of the general model.
\end{remark}

\section{QUBO Formulation of MaRP}
\label{sec:qubo}
To solve MaRP using optimization solvers, we encode the vehicle selection problem as a Quadratic Unconstrained Binary Optimization (QUBO) model \cite{Punnen2022, Lewis2017}. In our formulation, the QUBO is represented through its diagonal and upper-triangular coefficients, so that each pairwise interaction is counted only once. Constructing a QUBO for MaRP involves translating the route-level coverage reward and overlap penalty into these coefficients.

\subsection{Encoding Coverage and Overlap}
We introduce one binary variable $x_i$ for each vehicle $i =1, \ldots, n$. To rewrite the objective from \eqref{obj} in QUBO form, we define the QUBO matrix $\mathbf{Q} \in \mathbb{R}^{n \times n}$ with entries:
\begin{itemize}
\item $Q_{ii} = -u_i$ for each $i=1,\dots,n$. This sets a negative linear coefficient for selecting vehicle $i$, representing the route utility benefit. In the QUBO minimization view, selecting route $i$ contributes $Q_{ii} x_i = -u_i$ (a negative cost) to the objective, which encourages its selection.
\item $Q_{ij} = \lambda c_{ij}$ for each pair $i<j$. This sets a positive quadratic coupling between $i$ and $j$ equal to $\lambda$ times their overlap. If both $x_i$ and $x_j$ are 1, the term $Q_{ij}x_i x_j = \lambda c_{ij}$ adds a positive cost $\lambda c_{ij}$, penalizing the joint selection of two overlapping routes.
\end{itemize}
All other entries are zero. Given this $\mathbf{Q}$, the QUBO objective function to be minimized is
\begin{equation}
f(\mathbf{x})
=
\sum_{i=1}^{n} Q_{ii}\,x_i
+
\sum_{1 \le i < j \le n} Q_{ij}\,x_i x_j \, .
\end{equation}

Substituting the $\mathbf{Q}$ entries, this becomes 

\begin{equation}
\label{eq:qubo_min}
    f(\mathbf{x}) = -\sum_{i} u_i x_i + \lambda \sum_{i<j} c_{ij} x_i x_j .
\end{equation}
Minimizing $f(\mathbf{x})$ is directly equivalent to maximizing the original objective in Eq.~\eqref{obj} and the problem is now encoded as a standard QUBO: find $\mathbf{x} \in \{0,1\}^n$ that minimizes $f(\mathbf{x})$.

\subsection{Penalty Weights: $\lambda_{\text{soft}}$ vs. $\lambda_{\text{hard}}$}
\label{subsec:penalties}

The penalty parameter $\lambda$ plays a crucial role in the QUBO model, as it governs the balance between maximizing route utility and minimizing overlap. We consider two regimes for penalty parameter $\lambda$.\\

\noindent
\textbf{Hard penalty ($\lambda_{\text{hard}}$):} Here $\lambda$ is set to an extremely large value, effectively turning the overlap penalty into a hard constraint. This mirrors the theoretical reduction in Theorem~\ref{thm:nphard}, where we chose $\lambda$ greater than the sum of all coverage rewards to forbid overlapping routes in any optimal solution, 
\begin{equation}
    \lambda_{\text{hard}} = 1 + \sum_i u_i .
\end{equation}
With $\lambda_{\text{hard}}$ the QUBO solver will strongly favor solutions with no pairwise overlaps, even if that means sacrificing considerable coverage. \\

\noindent
\textbf{Soft penalty ($\lambda_{\text{soft}}$):} In contrast, a moderate or small $\lambda$ treats overlap as a soft constraint, allowing some redundancy if it significantly increases route utility. Selecting the right $\lambda_{\text{soft}}$ is non-trivial. We therefore use a scale-aware heuristic that compares a typical route reward with a typical total overlap level. Let
\begin{equation}
s_i=\sum_{j\neq i} c_{ij}
\end{equation}
denote the aggregate pairwise overlap of route $i$. We define
\begin{equation}\label{eq:lambdasoft}
\lambda_{\mathrm{soft}} =
\frac{\operatorname{median}_i u_i}
{\max\{1,\operatorname{median}_i s_i\}}.
\end{equation}
In practical terms, this adaptive choice of $\lambda_{\text{soft}}$ yields solutions that are well balanced with respect to the scale of the underlying instance, avoiding both overly permissive and overly restrictive behavior. It allows the solver to naturally trade small amounts of overlap for meaningful coverage gains when beneficial. From a broader perspective, this formulation can be viewed as a scalarized multiobjective optimization problem \cite{Ehrgott2005}, where coverage maximization and overlap minimization represent competing objectives combined through the penalty parameter.

\section{Simulation Setup and Metrics}
\label{sec:simulation}
To test the QUBO-based approach in a realistic setting, we built a pipeline that generates a city network, route instances, and formulates the corresponding QUBO, which is then solved using different solvers. To evaluate the resulting solutions, we use a set of metrics that capture not only overall coverage and route overlap, but also structural properties of the road network induced by the selected vehicles. 

\subsection{Instance Generation and Implementation}
\label{subsec:pipeline}
Road networks are extracted from OpenStreetMap using OSMnx \cite{Boeing2017}.
For each experiment, a city region of Barcelona is defined by a center and radius, the network graph is downloaded, and stored in a MariaDB database for reproducibility.
Vehicles are generated by sampling origins and destinations on the network. 
The number of vehicles varies across experiments to study scalability, ranging from 100 to 10,000 vehicles. Smaller instances are generated within a radius of 1 km around the city center, while larger instances use an expanded radius of 5 km to accommodate the increased number of vehicles without excessive route congestion.
Each vehicle is assigned one route extracted from the Valhalla routing engine \cite{valhalla}, 
with route represented as a set of network nodes. The node-to-route indexing is used to compute route-level rewards $u_i$ and pairwise overlap coefficients $c_{ij}$.
The resulting QUBO is then solved for different $\lambda$ regimes, and solutions from all solvers together with the associated evaluation metrics are stored for each simulation run.\footnote{The implementation is written in Python and maintained in a GitHub repository; access can be provided upon reasonable request.}

\subsection{Solvers}
\label{subsec:solvers}

Having formulated MaRP as a QUBO, we can leverage different solvers to obtain optimal or near-optimal solutions. In this study, we consider three representative solution approaches commonly used for WSP and related NP-hard combinatorial optimization problems.

\begin{enumerate}
    \item \textbf{Gurobi binary quadratic-programming baseline.}  
    The QUBO objective Eq.~\eqref{eq:qubo_min} is solved as a binary quadratic program using the Gurobi Optimizer \cite{gurobi}. Computations are performed using Gurobi's cloud-based execution environment accessed through the Python \texttt{gurobipy} interface. For each instance, the solver runtime is capped at a fixed wall-clock limit of 10 minutes. If optimality is proven within this limit, the returned solution is globally optimal; otherwise, the best feasible solution available at termination is used together with the corresponding solver status and optimality gap when available. 
    \item \textbf{Simulated annealing (SA).}  
    Simulated annealing is employed as a classical heuristic method based on stochastic local search with a cooling schedule \cite{Kirkpatrick1983}. We use D-Wave's open-source \texttt{neal} library with default parameters (\texttt{num\_reads = 100}, \texttt{sweeps = 1000}, and automatically determined annealing schedule). SA is executed locally on a workstation equipped with an AMD Ryzen~5~5500U CPU (2.10\,GHz) and 16\,GB RAM. No explicit runtime limit is imposed; each run terminates after completion of the prescribed number of sweeps.
    \item \textbf{Hybrid quantum annealing (QA).}  
    Hybrid quantum annealing is performed using D-Wave's cloud-based hybrid BQM solver via the Ocean SDK \cite{DWaveHybrid}. The MaRP QUBO is submitted directly in Binary Quadratic Model (BQM) form. The solver combines classical decomposition techniques with quantum annealing subroutines to handle large QUBO instances. All runs are executed with default solver parameters, including the default time limit and internal decomposition strategy selected automatically by the solver. The underlying quantum hardware is based on the Advantage2 system with Zephyr topology, which offers high qubit connectivity \cite{DWaveAdvantage2}. Although minor embedding is managed internally within the hybrid framework, embedding-related overhead remains a practical consideration for dense logical QUBOs \cite{Choi2008}.

\end{enumerate}

\subsection{Post-optimization Evaluation Metrics}
\label{subsec:metics}

To evaluate solution quality for MaRP, we use a set of metrics that quantify both network coverage and route overlap under different penalty settings. These metrics, computed from node-level usage statistics, characterize how effectively and evenly the selected vehicle routes cover the city network.

Let $N$ denote the set of network nodes visited by at least one selected route, and let $u_v$ be the number of vehicles whose routes pass through node $v$.

\paragraph{Average node overlap.}
To characterize the overall level of overlap in the network, we measure the average number of vehicles traversing a node. The average node overlap is defined as
\begin{equation}
\bar{u} = \frac{1}{|N|} \sum_{v \in N} u_v .    
\end{equation}
This metric provides a concise measure of redundancy in the selected routes. Values close to one indicate minimal overlap, whereas larger values reflect increasing concentration of traffic on shared network nodes.

\paragraph{Top 10\% overlap share.}
Nodes are ranked in descending order according to their overlap values. 
Let $K = \lceil 0.1\,|N| \rceil$, and let $N_K \subseteq N$ denote the set of the top 10\% of nodes (i.e., the $K$ nodes) with the highest values of $u_v$. The top 10\% overlap share is defined as
\begin{equation}
S_{10} = \frac{\sum_{v \in N_K} u_v}{\sum_{v \in N} u_v}.
\end{equation}
This measures the fraction of total route overlap contributed by the most frequently traversed nodes. Lower values indicate a more evenly distributed overlap across the network.

\paragraph{Overlap graph density and average degree.}
An overlap graph is constructed in which vertices correspond to selected vehicle routes and edges indicate shared nodes. If the graph, modeled as an undirected simple graph, contains $n_s$ selected-route vertices and $m_s$ edges, its density and average degree are
\begin{equation}
\delta = \frac{2m_s}{n_s(n_s-1)}, \qquad \bar{d} = \frac{2m_s}{n_s}.
\end{equation}

\paragraph{Normalized Shannon entropy~\cite{Shannon1948}.}
A probability distribution over network nodes is obtained by normalizing the node overlap counts,
\begin{equation}
p_v = \frac{u_v}{\sum_{v' \in N} u_{v'}}.
\end{equation}
The normalized Shannon entropy is defined as
\begin{equation}
H_{\mathrm{norm}} =
\frac{-\sum_{v \in N} p_v \log(p_v)}{\log(|N|)}.
\end{equation}
Values close to $1$ indicate that overlap is distributed nearly uniformly across the network, whereas smaller values reflect increasing concentration on a limited subset of nodes.

\paragraph{Herfindahl--Hirschman Index (HHI)~\cite{Brezina2016}.}
Overlap concentration is also quantified using the Herfindahl--Hirschman Index,
\begin{equation}
\mathrm{HHI} = \sum_{v \in N} p_v^2.
\end{equation}
Larger values of HHI indicate stronger concentration of overlap on a small number of nodes, while smaller values correspond to a more evenly distributed use of the network.

\paragraph{Coverage percentage.}
Coverage is measured as the fraction of network nodes covered after optimization relative to a pre-optimization configuration
\begin{equation}
\mathrm{Pct}_{\mathrm{cov}} = 100 \cdot 
\frac{|N_{\text{selected}}|}{|N_{\text{pre-optimization}}|},
\end{equation}
where $N_{\text{selected}} = \{ v : u_v > 0 \}$ denotes the set of nodes covered by the selected vehicles (corresponding to the set $N$ defined above), while $N_{\text{pre-optimization}}$ corresponds to the set of nodes covered when all candidate routes are considered. Values close to $100\%$ indicate that most of the original coverage is preserved.

\paragraph{Overlap percentage.}
Overlap is quantified as the ratio of total route overlap after optimization to that observed in the pre-optimization configuration
\begin{equation}
\mathrm{Pct}_{\mathrm{ov}} = 100 \cdot 
\frac{\mathrm{overlap}_{\text{selected}}}{\mathrm{overlap}_{\text{pre-optimization}}}.
\end{equation}
Here, $\mathrm{overlap}_{\text{selected}}$ denotes the total pairwise overlap between selected routes, computed as the sum of $c_{ij}$ over all pairs of selected routes, while $\mathrm{overlap}_{\text{pre-optimization}}$ represents the corresponding quantity when all routes are included. Lower values indicate a stronger reduction of overlap relative to the baseline.

\paragraph{Selected vehicle percentage.}
The proportion of vehicles retained by the optimization is given by
\begin{equation}
\mathrm{Pct}_{\mathrm{veh}} = 100 \cdot \frac{n_{\text{selected}}}{n_{\text{total}}},
\end{equation}
where $n_{\text{selected}}$ denotes the number of vehicles selected in the optimized solution and $n_{\text{total}}$ is the total number of candidate vehicles in the instance.

\paragraph{Objective value (energy).}
The objective value corresponds to the QUBO energy evaluated at the solution returned by a given solver for a fixed value of $\lambda$, capturing the trade-off between coverage rewards and overlap penalties encoded in the model.

\section{Experimental Results}
\label{sec:results}

We begin by comparing the behavior of the considered solvers in terms of runtime and solution quality. 

\paragraph{Solver performance comparison:}
Table~\ref{tab:solver_comparison} compares solver runtimes and objective values across increasing fleet sizes. The hybrid QA solver shows notably stable runtimes for small and medium instances, remaining close to $3\,\mathrm{s}$ up to $1{,}000$ vehicles and increasing gradually as the problem size grows. In contrast, simulated annealing (SA) exhibits poor scalability, with runtimes rising sharply and exceeding one hour already at $2{,}000$ vehicles. Gurobi provides the fastest solutions for small instances, achieving sub-second runtimes, but its computational cost increases steadily and becomes comparable to hybrid QA for larger fleets.

With respect to solution quality, hybrid QA consistently attains objective values identical or very close to those obtained by Gurobi across all tested scales. Simulated annealing performs competitively only for the smallest problems, and as the instance size increases, both its runtime and solution quality deteriorate significantly. Overall, the results indicate that both the hybrid QA approach and Gurobi are well-suited for this class of problems: Gurobi provides exact solutions efficiently at smaller scales, where the reported optimality gap is zero, while hybrid QA maintains high solution quality with stable runtimes as problem size increases.

Based on the results, all metrics reported in this section are computed from solutions obtained with Gurobi, which serves as the optimization baseline. The corresponding solutions obtained with the hybrid QA solver were nearly identical across all tested instances in terms of objective value and resulting coverage--overlap trade-offs.

Having established solver equivalence in terms of objective quality, we now analyze the coverage--overlap trade-off induced by different values of the penalty parameter $\lambda$ using a Pareto-frontier perspective.\\

\end{multicols}

\btab
\centering
\footnotesize
\caption{Comparison of solver performance for Barcelona scenarios across different fleet sizes. Reported values are averaged over all runs and penalty parameters.}
\label{tab:solver_comparison}
\begin{tabular}{rcccccc}
\toprule
$n_{\text{veh}}$
& \multicolumn{3}{c}{Solver time (s)}
& \multicolumn{3}{c}{Objective value / Energy} \\
\cmidrule(lr){2-4}
\cmidrule(lr){5-7}
 & QA (hybrid) & SA & Gurobi
 & QA energy & SA energy & Gurobi obj. \\
\midrule
100    & 2.99  & 5.33      & 0.03  & $-3{,}063.92$    & $-3{,}063.92$    & $-3{,}063.92$ \\
200    & 2.99  & 18.51     & 0.11  & $-7{,}955.21$    & $-7{,}934.99$    & $-7{,}955.21$ \\
300    & 2.99  & 43.95     & 1.16  & $-16{,}561.49$   & $-15{,}957.58$   & $-16{,}561.49$ \\
500    & 2.99  & 187.70    & 1.19  & $-21{,}291.39$   & $-20{,}123.44$   & $-21{,}291.39$ \\
1,000  & 2.99  & 1{,}364.90 & 0.67  & $-25{,}998.28$   & $-24{,}665.73$   & $-25{,}998.28$ \\
2,000  & 5.23  & 3{,}728.31 & 2.65  & $-137{,}607.09$  & $-122{,}340.09$  & $-137{,}607.09$ \\
3,000  & 7.50  & 10{,}525.78& 4.09  & $-164{,}370.46$  & $-148{,}763.32$  & $-164{,}370.46$ \\
5,000  & 14.56 & --        & 11.19 & $-172{,}311.08$  & --               & $-172{,}311.08$ \\
7,000  & 24.69 & --        & 17.44 & $-703{,}610.19$  & --               & $-703{,}609.94$ \\
10,000 & 32.44 & --        & 35.12 & $-985{,}012.34$  & --               & $-985{,}012.34$ \\
\bottomrule
\end{tabular}
\etab

\bfig
    \centering
    \includegraphics[width=0.70\linewidth]{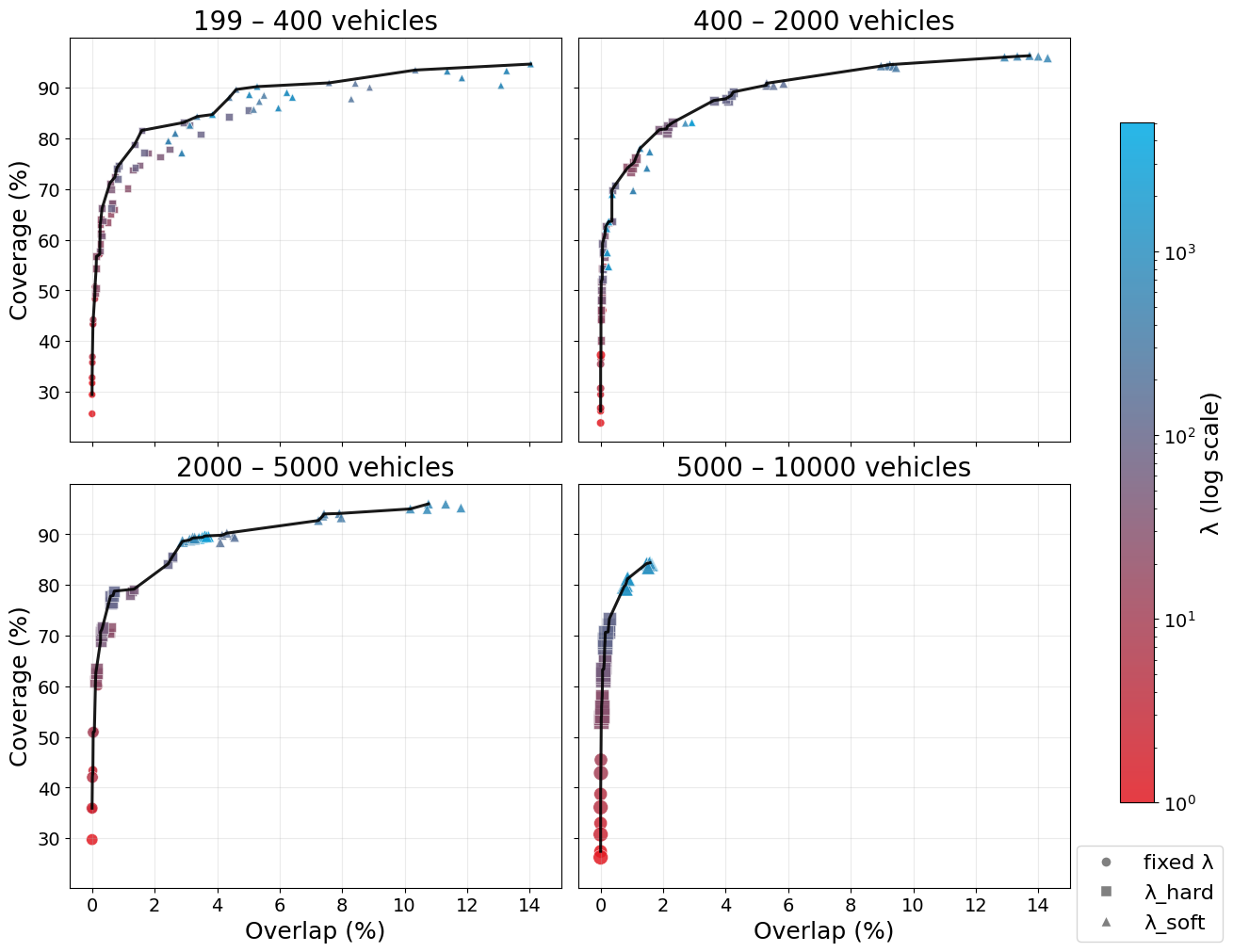}
    \caption{Coverage vs. Overlap with scale-specific Pareto frontiers.}
    \label{fig:coverage_overlap_pareto}
\efig

\begin{multicols}{2}

\paragraph{Pareto analysis of coverage and overlap:}
MaRP is inherently a bi-objective optimization problem, trading off network coverage against route overlap. In multi-objective optimization, a solution is considered \emph{Pareto optimal} if no other feasible solution can improve one objective without worsening at least one other objective \cite{Ehrgott2005}. The set of all non-dominated, Pareto-optimal solutions forms the Pareto frontier, representing the optimal trade-offs between conflicting objectives \cite{Giannelos2024}. In the context of MaRP, the parameter $\lambda$ acts as a weight that balances route utility and overlap in a single objective, effectively selecting one solution from the \emph{Pareto frontier} according to how strongly overlap reduction is prioritized. Consequently, there is no universally optimal value of $\lambda$; its choice reflects application-specific preferences.

Fig.~\ref{fig:coverage_overlap_pareto} summarizes the coverage--overlap trade-off stratified by problem scale, where instances are grouped into bins according to fleet size and a Pareto frontier is computed separately within each bin. Across all problem scales, the Pareto frontiers exhibit a pronounced knee: coverage improves rapidly up to approximately $75$–$85\%$ with minimal overlap, after which additional coverage gains incur a disproportionately higher overlap cost. As the fleet size grows, the frontier shifts downward and becomes steeper, indicating increasing difficulty in achieving high coverage without redundancy.
A consistent pattern is visible in the composition of the Pareto-optimal sets. The boundary of the trade-off is dominated almost entirely by solutions obtained under the $\lambda_{\text{hard}}$ formulation. This behavior follows directly from the definition of $\lambda_{\text{hard}}$ (Subsection~\ref{subsec:penalties}), which treats overlap as an almost hard constraint. As the problem size increases, such strict penalization becomes increasingly important for obtaining low-overlap Pareto-optimal solutions.

\bfig
    \centering
    \includegraphics[width=1.0\linewidth]{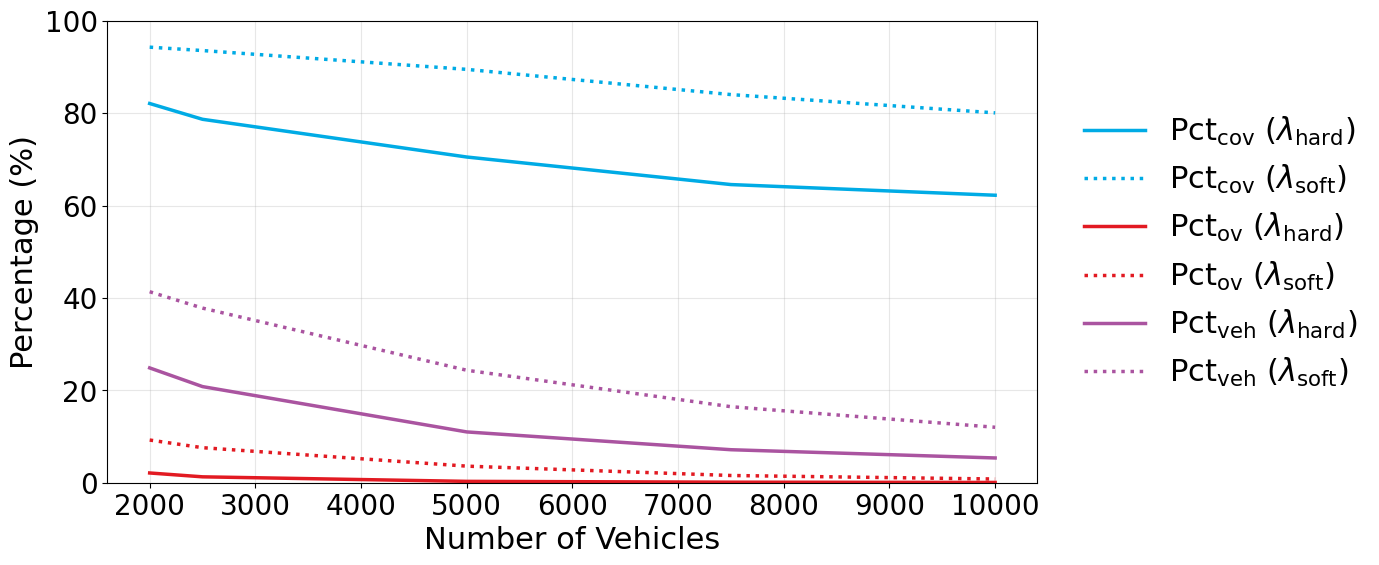}
    \caption{Impact of $\lambda$ on Coverage, Overlap, and Selected Vehicles (radius = 5 km)}
    \label{fig:pct}
\efig

\paragraph{Penalty regimes and scaling:}
Fig.~\ref{fig:pct} focuses on the large-scale scenario (radius $=5\,\mathrm{km}$) and illustrates how the three key performance indicators ($\mathrm{Pct}_{\mathrm{cov}}, \mathrm{Pct}_{\mathrm{ov}}, \mathrm{Pct}_{\mathrm{veh}}$) evolve as the fleet size increases under soft and hard penalty regimes. As the number of vehicles grows, coverage and the fraction of selected vehicles decrease steadily, as a larger number of vehicles must share the same road infrastructure, while overlap remains consistently low. The $\lambda_{\text{soft}}$ regime maintains higher coverage and assigns more vehicles across all scales, whereas $\lambda_{\text{hard}}$ achieves lower overlap at the cost of reduced coverage and vehicle selection. The overlap under $\lambda_{\text{hard}}$ stays close to zero even for very large fleets, again, indicating strong congestion control.

From a practical perspective, this behavior suggests a useful planning tool. Given a desired coverage level or an acceptable overlap threshold, the curves can be used to estimate how many vehicles can be deployed without overloading the network. In this sense, the proposed framework supports not only optimization but also capacity planning.

\paragraph{Structural and distributional traffic metrics:}
Beyond aggregate coverage and overlap, additional metrics reveal how traffic is structurally distributed across the network. The overlap graph density $\delta$ and average degree $\bar d$ indicate how strongly selected routes interact. In the small networks with radius $1\,\mathrm{km}$, $\delta$ typically ranges from $0.37$ to $0.43$, with $\bar d$ growing from roughly $80$ (200 vehicles) to more than $420$ (1,000 vehicles), whereas in the $5\,\mathrm{km}$ cases $\delta$ drops to around $0.06$ and $\bar d$ remains near $125$--$160$, reflecting sparser interactions.

Traffic concentration metrics further clarify how overlap is distributed spatially. For weak penalties, the normalized Shannon entropy $H_{\mathrm{norm}}$ remains close to one and HHI values are very small, indicating nearly uniform usage. As penalties strengthen, entropy decreases slightly, and HHI increases modestly, accompanied by higher average node overlap $\bar u$ and larger top 10\% overlap shares $S_{10}$. Importantly, these changes indicate controlled concentration rather than extreme congestion, confirming that improved coverage is achieved without collapsing traffic onto a small set of nodes.

In practical terms, this means that after optimization, the traffic is spread relatively evenly over the network. Stronger penalties introduce some concentration on high-utility nodes, but without creating dominant bottlenecks or congestion hotspots, indicating a well-balanced post-optimization network structure.

\section{Conclusion}
\label{sec:conclusion}
This paper formulated Multi-Agent Route Planning as a vehicle-route selection problem that maximizes route-level coverage utility of an urban road network while explicitly penalizing redundant overlaps between predefined routes. We provided a formal problem definition, established NP-hardness via a polynomial-time reduction from Weighted Set Packing for the general set-based formulation, and derived an interpretable QUBO model whose linear terms encode route utility rewards and whose quadratic couplings encode pairwise overlap penalties. By distinguishing soft and hard penalty regimes for the trade-off parameter $\lambda$, we showed how the same formulation supports both exploratory multi-objective analysis and near-disjoint selections. In a realistic pipeline based on OpenStreetMap networks and Valhalla-generated routes, we evaluated a Gurobi binary quadratic-programming baseline, simulated annealing, and D-Wave hybrid quantum annealing on Barcelona instances ranging up to $10{,}000$ vehicles. The experiments revealed a consistent coverage--overlap knee and indicated that Pareto-optimal solutions are predominantly achieved in the hard-penalty regime, while hybrid quantum annealing and Gurobi solvers attain very similar objective values on matching configurations with comparable scaling behavior at larger sizes.

At the same time, our formulation intentionally simplifies several aspects of real traffic systems. MaRP assumes one fixed route per vehicle, uses node-level overlap as a spatial proxy for potential congestion, and does not model temporal effects, interactions with signals, or capacity constraints. These simplifications are deliberate: they isolate the core combinatorial structure and enable scalable optimization, but they also highlight directions for extending the model toward richer, time-aware, and capacity-aware variants.

Practically, the proposed framework supports fleet sizing and dispatch decisions in logistics or ride-hailing. It can also serve as a decision-support tool for scenario analysis, allowing planners to quantify how many vehicles can be deployed for a desired coverage level under an acceptable redundancy threshold.

\section*{ACKNOWLEDGMENT}
This work was funded by the EU NextGenerationEU through the Recovery and Resilience Plan for Slovakia under project No.~17R05-04-V01-00004 (Competence and Training Center for Cyber and Information Security, TU~Košice). 

During the preparation of this article, the authors used an AI-based writing assistant to improve language and readability; the ideas and content presented are solely the authors' responsibility.

The implementation code is stored in a private GitHub repository. Access can be provided by the corresponding author upon reasonable request.


\end{multicols}
\end{document}